\documentclass{article} 
\usepackage{nips13submit_e,times}
\usepackage{hyperref}
\usepackage{url} 
  
\usepackage{graphicx}
\usepackage{natbib}
\usepackage{bm}
\usepackage{amsmath, amsthm, amssymb}

\usepackage{float}
\usepackage{ctable}
\usepackage{caption}
\usepackage{subcaption}

\title{Binary Classifier Calibration: Non-parametric approach}

\author{
Mahdi Pakdaman Naeini \\
Intelligent System Program\\
University of Piuttsburgh\\
\texttt{pakdaman@cs.pitt.edu} \\
\And
Gregory F. Cooper \\
Department of Biomedical Informatics \\
University of Pittsburgh \\
\texttt{gfc@pitt.edu} \\
\AND
Milos Hauskrecht \\
Computer Science Department \\
University of Pittsburgh \\
\texttt{milos@cs.pitt.edu} \\
}

\newcommand{\PP}{\mathbb{P}}

\nipsfinalcopy 
 
\begin{document}

\maketitle
 
\vspace{-0.08in}
\section*{Abstract}
Accurate calibration of probabilistic predictive models learned is
critical for many practical prediction and decision-making tasks. There are two
main categories of methods for building calibrated classifiers.
One approach is to develop methods for learning probabilistic models that are
well-calibrated, \textit{ab initio}. The other approach is to use some
post-processing methods for transforming the output of a classifier to be
well calibrated, as for example histogram binning, Platt scaling, and isotonic
regression. One advantage of the post-processing approach is that it can be applied
to any existing probabilistic classification model that was constructed using
any machine-learning method.

In this paper, we first introduce two measures for evaluating how well a
classifier is calibrated. We prove three theorems showing that using a simple
histogram binning post-processing method, it is possible to make a classifier
be well calibrated while retaining its discrimination capability. Also, by casting the
histogram binning method as a density-based non-parametric binary
classifier, we can extend it using two simple non-parametric density
estimation methods. We demonstrate the performance of the proposed calibration
methods on synthetic and real datasets. Experimental
results show that the proposed methods either outperform or are comparable to
existing calibration methods.

\vspace{-0.08in}
\section{Introduction}
\vspace{-0.07in}

The development of accurate probabilistic prediction models from data is
critical for many practical prediction and decision-making tasks. Unfortunately,
the majority of existing machine learning and data mining models and algorithms
are not optimized for this task and predictions they produce may be
miscalibrated.

Generally, a set of predictions of a binary outcome is well calibrated if the outcomes predicted to occur with
probability $p$ do occur about $p$ fraction of the time, for each probability
$p$ that is predicted. This concept can be readily generalized to outcomes with
more than two values. Figure \ref{IMG:Fig1} shows a hypothetical example of a reliability
curve \cite{degroot1983comparison,niculescu2005predicting}, which displays the
calibration performance of a prediction method. The curve shows, for example,
that when the method predicts $Z=1$ to have probability $0.5$, the outcome $Z=1$
occurs in about $0.57$ fraction of the instances (cases). The curve indicates
that the method is fairly well calibrated, but it tends to assign probabilities
that are too low. In general, perfect calibration corresponds to a straight line
from $(0, 0)$ to $(1, 1)$. The closer a calibration curve is to this line, the
better calibrated is the associated prediction method.

\begin{figure}[h] 
\begin{minipage}{16pc}
\centering 
  		\includegraphics[scale= 0.2]{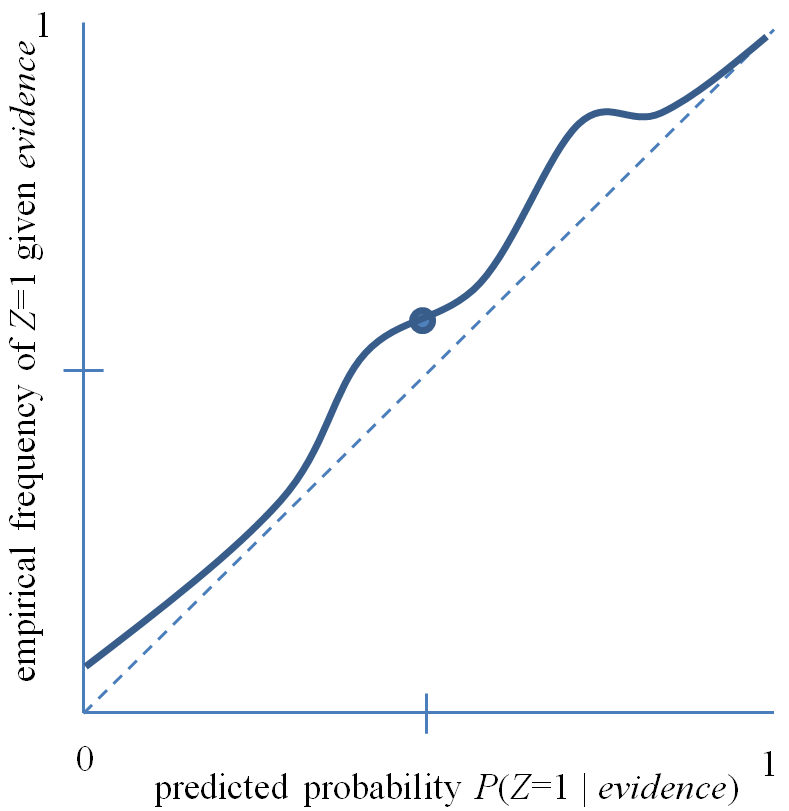}
  		\caption{The solid line shows a calibration (reliability) curve for predicting
		$Z=1$. The dotted line is the ideal calibration curve.}
        \label{IMG:Fig1}
\end{minipage}\hspace{2pc}%
\begin{minipage}{16pc}
   	\centering
  		\includegraphics[scale=0.4]{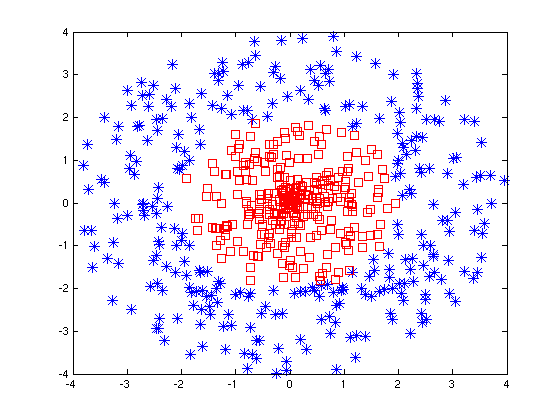}
  		\caption{Scatter plot of non-linear separable simulated data }
        \label{fig:SimulationData}

\end{minipage}
\end{figure}

If uncertainty is represented using probabilities, then optimal decision making
under uncertainty requires having models that are well calibrated. Producing
well calibrated probabilistic predictions is critical in many areas of science
(e.g., determining which experiments to perform), medicine (e.g., deciding which
therapy to give a patient), business (e.g., making investment decisions), and
many other areas. At the same time, calibration has not been studied nearly as
extensively as discrimination (e.g., ROC curve analysis) in machine learning and
other fields that research probabilistic modeling.

One approach to achieve a high level of calibration is to develop methods for
learning probabilistic models that are well-calibrated, \textit{ab initio}.
However, data mining and machine learning research has traditionally focused on
the development of methods and models for improving discrimination, rather than
on methods for improving calibration. As a result, existing methods have the
potential to produce models that are not well calibrated. The miscalibration
problem can be aggravated when models are learned from small-sample data or when
the models make additional simplifying assumptions (such as linearity or
independence).

Another approach is to apply post-processing methods (e.g., histogram
binning, Platt scaling, or isotonic regression) to the output of classifiers to
improve their calibration. The post-processing step can be seen as a function that maps output
of a prediction model to probabilities that are intended to be well calibrated.
Figure \ref{IMG:Fig1} shows an example of such a mapping. This approach frees
the designer of the machine learning model from the need to add additional
calibration measures and terms into the objective function used to learn a model.
The advantage of this approach is that it can be used with any existing
classification method, since calibration is performed
solely as the post-processing step. 

The objective of the current paper is to show that
the post-processing approach for calibrating binary classifiers is
theoretically justified. In particular, we show in the large sample limit that
post-processing will produce a perfectly calibrated classifier that has
discrimination perform (in terms of area under the ROC curve) that is at least
as good as the original classifier.
In the current paper we also introduce  two simple but effective methods
that can address the miscalibration problem. 

Existing post-processing calibration methods can be divided into parametric and non-parametric
methods. An example of a parametric method is Platt's method that applies a
sigmoidal transformation that maps the output of a predictive model
\cite{platt1999probabilistic} to a calibrated probability output. The parameters
of the sigmoidal transformation function are learned using a maximum
likelihood estimation framework. The key limitation of the approach is the
(sigmoidal) form of the transformation function, which only rarely fits the true
distribution of predictions.

The above problem can be alleviated using non-parametric methods. The most
common non-parametric methods are based either on binning
\cite{zadrozny2001obtaining} or isotonic regression \cite{ayer1955empirical}. In
the histogram binning approach, the raw predictions of a binary classifier are
sorted first, and then they are partitioned into $b$ subsets of equal
size, called bins. Given a prediction $y$, the method finds the bin
containing that prediction and returns as $\hat{y}$ the fraction of positive
outcomes $(Z=1)$ in the bin.

Zadronzy and Elkan \cite{zadrozny2002transforming} developed a calibration
method that is based on isotonic regression. This method only requires that the
mapping function be isotonic (monotonically increasing)
\cite{niculescu2005predicting}. The \textit{pair adjacent violators} (PAV)
algorithm is one instance of an isotonic regression algorithm
\cite{ayer1955empirical}. The isotonic calibration method based on the (PAV)
algorithm can be viewed as a binning algorithm where the position of boundaries
and the size of bins are seleted according to how well the classifier ranks the
examples in the training data \cite{zadrozny2002transforming}. Recently a
variation of the isotonic-regression-based calibration method was described for
predicting accurate probabilities with a ranking loss\cite{menon2012predicting}.

In this paper, section \ref{section:notations} introduces two measures,
\textit{maximum calibration error}(MCE) and \textit{expected calibration
error}(ECE), for evaluating how well a classifier is calibrated. In section
\ref{section:theorem} we prove three theorems to
show that by using a simple histogram-binning calibration method, it is possible
to improve the calibration capability of a classifier measured in terms of $MCE$ and
$ECE$ without sacrificing its discrimination capability measured in terms of
\textit{area under (ROC) curve} ($AUC$). Section \ref{section:methods} introduces
two simple extensions of the histogram binning method by casting the method as a simple
density based non-parametric binary classification problem. The results of 
experiments that evaluate various calibration methods are presented in section
\ref{section:experiments}. Finally, section \ref{section:conclusion} states
conclusions and describes several areas for the future work.

\vspace{-0.08in}
\section{Notations and Assumptions}
\vspace{-0.05in}

\label{section:notations}
This section present the notation and assumptions we use for formalizing the
problem of calibrating a binary classifier. We also define two measures for
assessing the calibration of such classifiers.

Assume a binary classifier is defined as a mapping $\phi: R^d \rightarrow
[0,1]$. As a result, for every input instance $x\in R^d$ the output of the
classifier is $y = \phi(x)$ where $y \in [0,1]$. For calibrating the classifier
$\phi(.)$ we assume there is a training set $\{(x_i,y_i,z_i)\}_{i=1}^N$ where $x_i \in R^d$ is the i'th
instance and $y_i = \phi(x_i) \in [0,1]$, and $z_i \in \{0,1\}$ is the true
class of i'th instance. Also we define $\hat{y}_i$ as the probability estimate
for instance $x_i$ achieved by using the histogram binning calibration method, which
is intended to yield a more calibrated estimate than does $y_i$. In addition we have
the following notation and assumptions that are used in the remainder of the paper:

\begin{itemize}
  \item $N$ is total number of instances
  \item $m$ is total number of positive instances
  \item $n$ is total number of negative instances
  \item $p_{in}$ is the space of uncalibrated probabilities $\{y_i\}$ which is
  defined by the classifier output
  \item $p_{out}$ is the space of transformed probability estimates $\{\hat{y}_i\}$ using histogram binning
  \item $B$ is the total number of bins defined on $p_{in}$ in the histogram
  binning model
  \item $B_i$ is the i'th bin defined on $p_{in}$
  \item $N_i$ is total number of instances $x_k$ for which the
  predicted value $y_k$ is located inside $B_i$
  \item $m_i$ is number of positive instances $x_k$ for which the
  predicted value $y_k$ is located inside  $B_i$
  \item $n_i$ is number of negative instances $x_k$ for which the
  predicted value $y_k$ is located inside $B_i$
  \item $\hat{\eta}_i = \frac{N_i}{N}$ is an empirical estimate of
  $\PP \{y \in B_i\}$
  \item $\eta_i$ is the value of $\PP \{y \in B_i\}$ as $N$ goes to infinity
  \item $\hat{\theta}_i = \frac{m_i}{N_i}$ is an empirical estimate of
  $\PP \{z=1| y \in B_i\}$
  \item $\theta_i$ is the value of $\hat{\theta}_i $ as $N$ goes to infinity
\end{itemize}
 
\subsection{Calibration Measures}
\label{section:measures}
In order to evaluate the calibration capability of a classifier, we use two
simple statistics that measure calibration relative to the ideal reliability diagram
\cite{degroot1983comparison,niculescu2005predicting}(Figure \ref{IMG:Fig1}
shows an example of a reliability diagram). These measures
are called Expected Calibration Error (ECE), and Maximum Calibration Error (MCE). In computing
these measures, the predictions are sorted and partitioned into ten bins. The
predicted value of each test instance falls into one of the bins. The $ECE$
calculates Expected Calibration Error over the bins, and $MCE$ calculates the
Maximum Calibration Error among the bins, using empirical estimates as follows:
\begin{eqnarray*}
ECE =  \sum_{i=1}^{10} P(i) \cdot \left|o_i-e_i \right|  &,&  MCE =  \max
\left( \left| o_i-e_i \right| \right),
\end{eqnarray*}
where $o_i$ is the true fraction of positive instances in bin $i$, $e_i$ is the
mean of the post-calibrated probabilities for the instances in bin $i$, and
$P(i)$ is the empirical probability (fraction) of all instances that fall into
bin $i$. The lower the values of $ECE$ and $MCE$, the better is the calibration
of a model.
 
\vspace{-0.08in}
\section{Calibration Theorems}
\vspace{-0.04in}
\label{section:theorem}

In this section we study the properties of the histogram-binning calibration method.
We prove three theorems
that show that this method can improve the calibration power of a classifier
without sacrificing its discrimination capability. 

The first theorem shows that the $MCE$ of
the histogram binning method is concentrated around zero:

\newtheorem{theorem}{Theorem}[section]
\newtheorem{corollary}[theorem]{Corollary}
\newtheorem{lemma}{Lemma}[section]

\newenvironment{definition}[1][Definition]{\begin{trivlist}
\item[\hskip \labelsep {\bfseries #1}]}{\end{trivlist}}
\newenvironment{example}[1][Example]{\begin{trivlist}
\item[\hskip \labelsep {\bfseries #1}]}{\end{trivlist}}
\newenvironment{remark}[1][Remark]{\begin{trivlist}
\item[\hskip \labelsep {\bfseries #1}]}{\end{trivlist}}

\begin{theorem}
\label{th:mce}
Using histogram binning calibration, with probability at
least $1-\delta$ we have $MCE \leq \sqrt{\frac{2B\log{\frac{2B}{\delta}}}{N}}$.
\end{theorem}

\begin{proof}
For proving this theorem, we first use a concentration result for
$\hat{\theta}_i$. Using Hoeffding's inequality we have the following:

\begin{equation}
\PP\{ |\hat{\theta}_i - \theta| \geq \epsilon \} \leq 2
e^{\frac{-2N\epsilon^2}{B}} \label{eq:theta-concenteration}
\end{equation}

Let's assume $\tilde{B}_i$ is a
bin defined on the space of transformed probabilities $p_{out}$ for calculating
the $MCE$ of histogram binning method. Assume after using histogram binning over
$p_{in}$ (space of uncalibrated probabilities which is generated by the classifier $\phi$),
$\hat{\theta}_{i1},..,\hat{\theta}_{ik_i}$ will be mapped into $\tilde{B}_i$.
We define $o_i$ as the true fraction of positive instances in bin
$\tilde{B}_i$, and $e_i$ as the mean of the post-calibrated probabilities for
the instances in bin $\tilde{B}_i$. Using the notation defined in section
\ref{section:notations}, we can write $o_i$ and $e_i$ as follows:

\begin{equation}
o_i = \frac{\eta_{i1}\theta_{i1} + \ldots + \eta_{ik_i}\theta_{ik_i}}
			{\eta_{i1} + \ldots + \eta_{ik_i}}, \text{    } 
e_i = \frac{\eta_{i1}\hat{\theta}_{i1} + \ldots +\eta_{ik_i}\hat{\theta}_{ik_i}}
			{\eta_{i1} + \ldots + \eta_{ik_i}} \notag\\
\end{equation}
by defining $\alpha_{it} = \frac{\eta_{it}}{\eta_{i1} + \ldots +
\eta_{ik_i}}$ and  using the triangular inequality we have that:

\begin{equation}
|o_i-e_i| \leq  \sum_{t \in \{1,\ldots,k_i\}}
\alpha_{it}|\hat{\theta}_{it}-\theta_{it} | \text{   } 
\leq \text{  } max_{t \in \{1,\ldots,k_i\}} |\hat{\theta}_{it} - \theta_{it} |
\label{eq:convex}
\end{equation}

Using the above result and the concentration inequality
\ref{eq:theta-concenteration} for $\hat{\theta_i}$ we can conclude:

\begin{align}
\PP\{ |o_i-e_i| > \epsilon \} &\leq \PP \{ max_{t \in \{1,\ldots,k_i\}}
|\hat{\theta}_{it} - \theta_{it} | > \epsilon \} 
\leq 2 k_i e^{\frac{-N \epsilon^2}{2B}},
\label{eq:OneBin}
\end{align}
Where the last part is obtained by using a union bound and $k_i$ is the
number of bins on the space $p_{in}$ for which their calibrated probability 
estimate will be mapped into the bin $\tilde{B}_i$.

Using a union bound again over different bins like  $\tilde{B}_i$ 
defined on the space $p_{out}$, we achieve the following probability bound 
for $MCE$ over the space of calibrated estimates $p_{out}$ :
\begin{align}	
\PP \{ \max_{i=1}^B |o_i - e_i| \geq \epsilon \} \leq 
2(K_1 + \ldots + K_B)e^{\frac{-N \epsilon^2}{2B}} \Longrightarrow 
\PP \{ MCE \geq \epsilon \} \leq 
  2B e^{\frac{-N \epsilon^2}{2B}} \notag 
\end{align}

By setting $\delta = 2B e^{\frac{-N \epsilon^2}{2B}}$ we can show that with
probability $1-\delta$ the following inequality holds $MCE \leq
\sqrt{\frac{2B\log{\frac{2B}{\delta}}}{N}}$.
\end{proof}

\begin{corollary}
Using histogram binning calibration method, MCE converges to
zero with the rate of $O(\sqrt{\frac{B \log{B}}{N}})$.
\end{corollary}

Next, we prove a theorem for bounding the ECE of the histogram-binning
calibration method as follows:
\begin{theorem}
\label{th:ece}	
Using histogram binning calibration method, ECE converges to
zero with the rate of $O(\sqrt{\frac{B }{N}})$.
\end{theorem}

\begin{proof}
The proof of this theorem uses the concentration inequality \ref{eq:OneBin}.Due
to space limitations, the details of the proof is stated in the supplementary
part of the paper.
%
\end{proof}

The above two theorems show that we can bound the calibration error of a binary
classifier, which is measured in terms of $MCE$ and $ECE$, by using a
histogram-binning post-processing method. We next show that in addition to
gaining calibration power, by using histogram binning we are guaranteed not to
sacrifice discrimination capability of the base classifier $\phi(.)$ measured in
terms of $AUC$. Recall the definitions of $y_i$ and $\hat{y_i}$, where $y_i =
\phi(x_i)$ is the probability prediction of the base classifier $\phi(.)$ for
the input instance $x_i$, and $\hat{y}_i$ is the transformed estimate for
instance $x_i$ that is achieved by using the histogram-binning calibration
method.

We can define the $AUC\_Loss$ of the histogram-binning calibration method as:
 
\begin{definition}
\emph{($AUC\_Loss$)}
$AUC\_Loss$ is the difference between the AUC of the base classifier
estimate and the AUC of transformed estimate using the histogram-binning calibration
method. Using the notation in Section \ref{section:notations}, it is defined as
$AUC\_Loss = AUC(y) - AUC(\hat{y})$
\end{definition}

Using the above definition, our third theorem bounds the $AUC\_Loss$ of
histogram binning classifier as follows:
\begin{theorem}
\label{th:auc}
Using the histogram-binning calibration method, the worst case
$AUC\_Loss$ is upper bounded by $O(\frac{1}{B})$.
\end{theorem}

\begin{proof}
Due to space limitations, the proof of this theorem is stated in the appendix
section in the supplementary part of the paper.
\end{proof}

Using the above theorems, we can conclude that by using the histogram-binning
calibration method we can improve calibration performance of a classifier measured in
terms of $MCE$ and $ECE$ without losing discrimination performance of the base
classifier measured in terms of $AUC$. 

We will show in Section \ref{section:methods} that the histogram binning
calibration method is simply a non-parametric plug-in classifier. By
casting histogram binning as a non-parametric histogram binary classifier, there
are other results that show the histogram classifier is a mini-max rate
classifier for Lipschitz Bayes decision boundaries
\cite{devroye1996probabilistic}. Although the results are valid for histogram
classifiers with fixed bin size, our experiments show that both fixed bin
size and fixed frequency histogram classifiers behave quite similarly. We conjecture
that a histogram classifier with equal frequency binning is also a mini-max (or
near mini-max) rate classifier\cite{scott2003near,klemela2009multivariate}; this
is an interesting open problem that we intend to study in the future.
These results make histogram binning a reasonable choice for binary
classifier calibration under the condition $B \rightarrow \infty$ and
$\frac{N}{B \log{B}} \rightarrow \infty$ as $N \rightarrow \infty$. This could
be achieved by setting $B \simeq N^{\frac{1}{3}}$, which is the optimum number
of bins in order to have optimal convergence rate results for the non-parametric
histogram classifier \cite{devroye1996probabilistic}.

\vspace{-0.08in} 
\section{Calibration Methods}
\vspace{-0.05in}

\label{section:methods} 
\interfootnotelinepenalty=10000
In this section, we show that the histogram-binning calibration method
\cite{zadrozny2001obtaining} is a simple nonparametric plug-in 
classifier. In the calibration problem, given an
uncalibrated probability estimate $y$, one way of finding the calibrated
estimate $\hat{y} = \PP(Z=1|y)$ is to apply Bayes' rule as follows:

\begin{equation}
 	\PP(Z=1|y) = \frac{P(z=1) \cdot P(y |z=1)} 
				{p(z=1) \cdot P(y |z=1)+ P(z=0)\cdot P(y |z=0)},
	\label{Eq:Bayes-Rule}
\end{equation} 
where $P(z=0)$ and $P(z=1)$ are the priors of class $0$ and $1$ that are
estimated from the training dataset. Also, $P(y|z=1)$ and
$P(y|z=0)$ are predictive likelihood terms. If we use the histogram
density estimation method for estimating the predictive likelihood terms in the Bayes
rule equation \ref{Eq:Bayes-Rule} we obtain the following: 
$\hat{P}(y | z = t) = \sum_{j=1}^{B} \frac{\hat{\theta}_j^t}{h_j} I(y \in B_j)$,
where $t=\{0,1\}$,  $\hat{\theta}_j^0 = \frac{1}{n} \sum_{i=1}^{N} I(y_i \in
B_j, z_i = 0)$, and $\hat{\theta}_j^1 = \frac{1}{m} \sum_{i=1}^{N} I(y_i \in
B_j, z_i = 1)$ are the empirical estimates of the probability of a prediction
when class $z = t$ falls into bin $B_j$. Now, let us assume $y \in
B_j$; using the assumptions in Section \ref{section:notations}, by
substituting the value of empirical estimates of $\hat{\theta_j^0}=
\frac{n_j}{n}$, $\hat{\theta_j^1}= \frac{m_j}{m}$, $\hat{P}(z=0)=\frac{n}{N}$,
$\hat{P}(z=1)=\frac{m}{N}$ from the training data and performing some basic algebra
we obtain the following calibrated estimate: $\hat{y} = \frac{m_j}{m_j +
n_j}$,
where $m_i$ and $n_j$ are the number of positive and negative examples in bin $B_j$.

The above computations show that the histogram-binning calibration method is actually a
simple plug-in classifier where we use the histogram-density method for
estimating the predictive likelihood in terms of Bayes rule as given by
\ref{Eq:Bayes-Rule}. By casting histogram binning as a plug-in method for
classification, it is possible to use more advanced frequentist methods
for density estimation rather than using simple histogram-based density
estimation. For example, if we use kernel density estimation (KDE) for estimating the predictive likelihood
terms, the resulting calibrated probability $P(Z=1|X=x)$ is as
follows:

\begin{equation}
\hat{P}(Z=1|X=x) = \frac{n h_0 \sum_{X_i \in X^+} K\left(\frac{|x-X_i|}{h_1}
\right)}
{n h_0 \sum_{X_i \in X^+} K\left(\frac{|x-X_i|}{h_1}
\right) +  m h_1 \sum_{X_i \in X^-} K\left(\frac{|x-X_i|}{h_0}
\right)},
\end{equation}

where $X_i$ are training instances, and $m$ and $n$ are
respectively the number of positive and negative examples in training
data. Also $h_0$ and $h_1$ are the bandwidth of the predictive likelihood for
class $0$ and class $1$. The bandwidth parameters can be optimized using cross
validation techniques. However, in this paper we used Silverman's rule of thumb
\cite{silverman1986density} for setting the bandwidth to $h = 1.06
\hat{\sigma} N^{-\frac{1}{5}} $, where $\hat{\sigma}$ is the empirical unbiased estimate of
variance. It is possible to use the same bandwidth for both class
$0$ and class $1$, which leads to the Nadaraya-Watson kernel estimator that we
use in our experiments. However, we noticed that there are some cases for which
KDE with different bandwidths performs better. 

There are different types of smoothing kernel functions, as the Gaussian, Boxcar,
Epanechnikov, and Tricube functions. Due to the similarity of the results we obtained when using
different type of kernels, we only report here the results of the simplest one, which
is the Boxcar kernel.

It has been shown in \cite{wasserman2006all} that kernel density estimators are
mini-max rate estimators, and under the $L_2$ loss function the risk
of the estimator converges to zero with the rate of $O_P
(n^{\frac{-2\beta}{(2\beta+d)} })$, where $\beta$ is a measure of smoothness of
the target density, and $d$ is the dimensionality of the input data.
From this convergence rate, we can infer that the application of kernel density
estimation is likely to be practical when $d$ is low. Fortunately, for the binary
classifier calibration problem, the input space of the model is the space of
uncalibrated predictions, which is a one-dimensional input space.
This justifies the application of KDE to the classifier
calibration problem.
 
The KDE approach presented above represents a
non-parametric frequentist approach for estimating the likelihood terms of
equation \ref{Eq:Bayes-Rule}.  Instead of using the frequentist approach, we can
use Bayesian methods for modeling the density functions. The Dirichlet Process
Mixture (DPM) method is a well-known Bayesian approach for density estimation
\cite{antoniak1974mixtures, ferguson1973bayesian,
escobar1995bayesian,maceachern1998estimating}. For building a Bayesian
calibration model, we model the predictive likelihood terms $P(X_i=x | Z_i=1)$
and $P(X_i=x | Z_i=0)$ in Equation \ref{Eq:Bayes-Rule} using the $DPM$ method.
Due to a lack of space, we do not present the details of the DPM model here, but
instead refer the reader to \cite{antoniak1974mixtures, ferguson1973bayesian,
escobar1995bayesian,maceachern1998estimating}.

There are different ways of performing inference in a $DPGM$ model. One
can choose to use either Gibbs sampling (non-collapsed or collapsed) or
variational inference, for example. In implementing our calibration model, we
use the variational inference method described in
\cite{kurihara2007accelerated}. We chose it because it
has fast convergence. We will refer to it as $DPM$.

\vspace{-0.08in}

\section{Empirical Results}
\vspace{-0.04in}

\label{section:experiments}

\begin{table*}[!t]
	\caption{Experimental Results on Simulated dataset}
    \begin{subtable}[b]{0.52\textwidth}
    \centering
 		  \tabcolsep 4.2pt
		  \caption{SVM Linear}
		  \resizebox{7cm}{!} {
				\tabcolsep=0.11cm 
				\begin{tabular}{@{}l@{}cccccc@{}}
			    \toprule
			      & SVM & Hist & Platt & IsoReg & KDE & DPM \\
			    \midrule
			    RMSE  & 0.50  & 0.39  & 0.50  & 0.46  & 0.38  & 0.39 \\
			    AUC   & 0.50  & 0.84  & 0.50  & 0.65  & 0.85  & 0.85 \\
			    ACC   & 0.48  & 0.78  & 0.52  & 0.64  & 0.78  & 0.78 \\
			    MCE   & 0.52  & 0.19  & 0.54  & 0.58  & 0.09  & 0.16 \\
			    ECE   & 0.28  & 0.07  & 0.28  & 0.35  & 0.03  & 0.07 \\
			    \bottomrule
			    \end{tabular}%
			 }
			\label{tab:SimulatedData_SVM_Linear}
	 \end{subtable}
	 \begin{subtable}[b]{0.52\textwidth} 
     	  \centering
 		  \tabcolsep 4.2pt
		  \caption{SVM Quadratic Kernel}
		  \resizebox{7cm}{!} {
			\tabcolsep=0.11cm
			\begin{tabular}{@{}l@{}cccccc@{}}
		    \toprule
		      & SVM & Hist & Platt & IsoReg & KDE & DPM \\
		    \midrule
		    RMSE  & 0.21  & 0.09  & 0.19  & 0.08  & 0.09  & 0.08 \\
		    AUC   & 1.00  & 1.00  & 1.00  & 1.00  & 1.00  & 1.00 \\
		    ACC   & 0.99  & 0.99  & 0.99  & 0.99  & 0.99  & 0.99 \\
		    MCE   & 0.35  & 0.04  & 0.32  & 0.03  & 0.07  & 0.03 \\
		    ECE   & 0.14  & 0.01  & 0.15  & 0.00  & 0.01  & 0.00 \\
		    \bottomrule
		    \end{tabular}%
		 }
		 \label{tab:SimulatedData_SVM_Quadratic}
	 \end{subtable}
	 \label{tab:SimulatedData}
\end{table*}

This section describes the set of experiments that we performed to evaluate the
performance of calibration methods described above. To evaluate the
calibration performance of each method, we ran experiments on both simulated
and on real data. For the evaluation of the calibration methods, we used $5$ different measures.
The first two measures are Accuracy (Acc) and the Area Under the ROC Curve
(AUC), which measure discrimination. The three other measures are the Root Mean
Square Error (RMSE), Expected Calibration Error (ECE), and Maximum Calibration
Error (MCE), which measure calibration.

{\bf Simulated data}. For the simulated data experiments, we used a binary
classification dataset in which the outcomes were not linearly separable. The
scatter plot of the simulated dataset is shown in Figure
\ref{fig:SimulationData}. The data were divided into $1000$ instances for
training and calibrating the prediction model, and $1000$ instances for testing the models.

To conduct the experiments on simulated datasets, we used two extreme
classifiers: support vector machines (SVM) with linear and quadratic
kernels. The choice of SVM with linear kernel allows us to see how the
calibration methods perform when the classification model makes over simplifying
(linear) assumptions. Also, to achieve good discrimination on the data in
figure \ref{fig:SimulationData}, SVM with quadratic kernel is
intuitively an ideal choice.  So, the experiment using quadratic kernel SVM
allows us to see how well different calibration methods perform when we use an
ideal learner for the classification problem, in terms of discrimination.

\begin{table*}[!t]
	\caption{Experimental results on size of calibration dataset}		 
     \begin{subtable}[b]{0.51\textwidth} 
     \centering
 		  \tabcolsep 4.2pt
		  \caption{SVM Linear}
		  \resizebox{7cm}{!} {
				 	 \begin{tabular}{@{}lccccc|c@{}}
						\toprule
				          & $10^2$     & $10^3$     & $10^4$     & $10^5$     & $10^6$     &
				           {\small Base SVM} \\ 
					    \midrule
					    AUC   & 0.82  & 0.84  & 0.85  & 0.85  & 0.85  & 0.49 \\
					    MCE   & 0.40  & 0.15  & 0.07  & 0.05  & 0.03  & 0.52 \\
					    ECE   & 0.14  & 0.05  & 0.03  & 0.02  & 0.01  & 0.28 \\
					   \bottomrule
				    \end{tabular}%
		  }
		  
	 \end{subtable}
	 \begin{subtable}[b]{0.51\textwidth} 
     \centering
 		  \tabcolsep 4.2pt
		  \caption{SVM Quadratic Kernel}
		  \resizebox{7cm}{!} {
				 	\begin{tabular}{@{}lccccc|c@{}}
				   		\toprule
				          & $10^2$     & $10^3$     & $10^4$     & $10^5$     & $10^6$     &
				          {\small Base SVM}  \\
					    \midrule
					    AUC   & 0.99  & 1.00  & 1.00  & 1.00  & 1.00  & 1.00 \\
					    MCE   & 0.14  & 0.09  & 0.03  & 0.01  & 0.01  & 0.36 \\
					    ECE   & 0.03  & 0.01  & 0.00  & 0.00  & 0.00  & 0.15 \\
					    \bottomrule
				    \end{tabular}%
		  }
		  
	 \end{subtable}
	 \label{tab:calibration-size}
\end{table*}

As seen in Table \ref{tab:SimulatedData}, KDE and DPM based calibration
methods performed better than Platt and isotonic regression in the simulation datasets,
especially when the linear SVM  method is used as the base learner. The poor
performance of Platt is not surprising given its simplicity, which consists of a
parametric model with only two parameters. However, isotonic regression is a
nonparametric model that only makes a monotonicity assumption over the output of
base classifier. When we use a linear kernel SVM, this assumption is
violated because of the non-linearity of data. As a 
result, isotonic regression performs relatively poorly, in terms of
improving the discrimination and calibration capability of a base classifier.
The violation of this assumption can happen in real data as well. In order to mitigate this pitfall,
Menon et. all \cite{menon2012predicting} proposed using a combination of
optimizing $AUC$ as a ranking loss measure, plus isotonic regression for
building a ranking model. However, this is counter
to our goal of developing post-processing methods that can be used with any
existing classification models. As shown in Table
\ref{tab:SimulatedData_SVM_Quadratic}, even if we use an ideal SVM classifier
for these linearly non-separable datasets, our proposed methods perform better
or as well as does isotonic regression based calibration. 

As can be seen in Table \ref{tab:SimulatedData_SVM_Quadratic}, although the SVM base
learner performs very well in the sense of discrimination based on AUC and Acc
measures, it performs poorly in terms of calibration, as measured by RMSE, MCE, and ECE.
Moreover, all of the calibration methods retain the same discrimination
performance that was obtained prior of post-processing, while improving
calibration. 

Also, Table \ref{tab:calibration-size} shows the results of experiments on using
the histogram-binning calibration method for different sizes of calibration sets on
the simulated data with linear and quadratic kernels. In these experiments we set
the size of training data to be $1000$ and we fixed $10000$ instances for testing
the methods. For capturing the effect of calibration size, we change the
size of calibration data from $10^2$ up to $10^6$, running the experiment $10$
times for each calibration set and averaging the results. As seen in Table
\ref{tab:calibration-size}, by having more calibration data, we have a steady decrease in
the values of the $MCE$ and $ECE$ errors.

{\bf Real data}. In terms of real data, we used a KDD-98 data set, which is
available at UCI KDD repository. 
The dataset  contains information about
people who donated to a particular charity. Here the decision making
task is to decide whether a solicitation letter should be mailed to a person
or not. The letter costs (which costs $ \$ 0.68 $). The training set includes $95,412$ instances
in which it is known whether a person made a donation, and if so, how much
the person donated. Among all these training cases, $4,843$ were
responders. The validation set includes $96,367$ instances from the same
donation campaign of which $4,873$ where responders. 

Following the procedure in
\cite{zadrozny2001obtaining,zadrozny2002transforming}, we build two models: a
\textit{response model} $r(x)$  for predicting the probability of responding to
a solicitation, and the \textit{amount model} $a(x)$ for predicting the amount
of donation of person $x$. The optimal mailing policy is to send a letter to those
people for whom the expected donation return $r(x) a(x)$ is greater than the
cost of mailing the letter. Since in this paper we are not concerned with
feature selection, our choice of attributes are based on
\cite{mayer2003experimental} for building the response and amount prediction
models. Following the approach in \cite{zadrozny2001learning}, we build the
amount model on the positive cases in the training data, removing the cases with
more than $ \$ 50 $ as outliers. Following their construction we also provide
the output of the response model $r(x)$ as an augmented feature to the amount model $a(x)$.

In our experiments, in order to build the response model, we used three
different classifiers: $SVM$, $Logistic Regression$ and $naive Bayes$. For building the amount model, we
also used a support vector regression model. For implementing these models we used
the liblinear package \cite{REF08a}. The results of the experiment are shown in
Table \ref{tab:KDD}. In addition to previous measures of comparison, we also
show the amount of profit obtained when using different methods.
As seen in these tables, the application of calibration methods results in at
least $ \$ 3000$ more in expected net gain from sending solicitations.

\begin{table*}
  \caption{Experimental Results on KDD 98 dataset}
  \begin{subtable}[b]{0.52\textwidth}
			  \centering
			  \tabcolsep 4.2pt
			  \caption{Logistic Regression}
			  \resizebox{7cm}{!} {
					\begin{tabular}{@{}lcccccc@{}}
				    \toprule
				    \multicolumn{1}{c}{} & \multicolumn{1}{c}{LR} & \multicolumn{1}{c}{Hist} & \multicolumn{1}{c}{Plat} & \multicolumn{1}{c}{IsoReg} & \multicolumn{1}{c}{KDE} & \multicolumn{1}{c}{DPM} \\
				    \midrule
				    RMSE  & 0.500 & 0.218 & 0.218 & 0.218 & 0.218 & 0.219 \\
				    AUC   & 0.613 & 0.610 & 0.613 & 0.612 & 0.611 & 0.613 \\
				    ACC   & 0.56  & 0.95  & 0.95  & 0.95  & 0.95  & 0.95 \\
				    MCE   & 0.454 & 0.020 & 0.013 & 0.030 & 0.004 & 0.017 \\
				    ECE   & 0.449 & 0.007 & 0.004 & 0.013 & 0.002 & 0.003 \\
				    Profit & 10560 & 13183 & 13444 & 13690 & 12998 & 13696 \\
				    \bottomrule
				    \end{tabular}%
			  }
			  \label{tab:KDD_LR}%
	\end{subtable}
    \begin{subtable}[b]{0.52\textwidth}
	  \centering
	  \tabcolsep 4.2pt
	  \caption{Na\"{i}ve Bayes }
		\resizebox{7cm}{!} {
				\begin{tabular}{@{}lcccccc@{}}
			    \toprule 
			          & \multicolumn{1}{c}{NB} & \multicolumn{1}{c}{Hist} & \multicolumn{1}{c}{Plat} & \multicolumn{1}{c}{IsoReg} & \multicolumn{1}{c}{KDE} & \multicolumn{1}{c}{DPM} \\
			    \midrule
			    RMSE  & 0.514 & 0.218 & 0.218 & 0.218 & 0.218 & 0.218 \\
			    AUC   & 0.603 & 0.600 & 0.603 & 0.602 & 0.602 & 0.603 \\
			    ACC   & 0.622 & 0.949 & 0.949 & 0.949 & 0.949 & 0.949 \\
			    MCE   & 0.850 & 0.008 & 0.008 & 0.046 & 0.005 & 0.010 \\
			    ECE   & 0.390 & 0.004 & 0.004 & 0.023 & 0.002 & 0.003 \\
			    Profit & 7885  & 11631 & 10259 & 10816 & 12037 & 12631 \\
			    \bottomrule
			    \end{tabular}%
	    }
	  \label{tab:KDD_NB}%
  \end{subtable}
  \begin{subtable}[b]{\textwidth}
		   \centering  
		  \tabcolsep 4.2pt
		  \caption{SVM Linear}
		  \resizebox{7.2cm}{!} {
				\begin{tabular}{@{}lcccccc@{}}
			    \toprule
			    \multicolumn{1}{c}{} & \multicolumn{1}{c}{SVM} & \multicolumn{1}{c}{Hist} & \multicolumn{1}{c}{Plat} & \multicolumn{1}{c}{IsoReg} & \multicolumn{1}{c}{KDE} & \multicolumn{1}{c}{DPM} \\
			    \midrule
			    RMSE  & 0.696 & 0.218 & 0.218 & 0.219 & 0.218 & 0.218 \\
			    AUC   & 0.615 & 0.614 & 0.615 & 0.500 & 0.614 & 0.615 \\
			    ACC   & 0.95  & 0.95  & 0.95  & 0.95  & 0.95  & 0.95 \\
			    MCE   & 0.694 & 0.011 & 0.013 & 0.454 & 0.003 & 0.019 \\
			    ECE   & 0.660 & 0.004 & 0.004 & 0.091 & 0.002 & 0.004 \\
			    Profit & 10560 & 13480 & 13080 & 11771 & 13118 & 13544 \\
			    \bottomrule
			    \end{tabular}%
		    }
		  \label{tab:KDD_SVM}%
  \end{subtable}
  \label{tab:KDD}
\end{table*}

\vspace{-0.04in}
\subsection{The Calibration Dataset}
\vspace{-0.04in}

In all of our experiments, we used the same training data for model calibration
as we used for model construction. In doing so, we did not notice any
over-fitting. However, if we want to be completely sure not to over-fit on the
training data, we can do one of the following:
 
\begin{itemize}
  \item \emph{Data Partitioning:} This approach uses different data sets for model training and model calibration.
  The amount of data that is needed to calibrate models is
  generally much less than the amount needed to train them, because the
  calibration feature space has a single dimension. We observed that
  approximately $1000$ instances are sufficient for obtaining well calibrated
  models, as is seen in table [\ref{tab:calibration-size}].
  \item \emph{Leave-one-out: } If the amount of available training data is small,
  and it not possible to do data partitioning, we can use a leave-one-out (or k-fold
  variation) scheme for building the calibration dataset. In this approach we
  learn a model based on $N-1$ instances, test it on the one remaining 
  instance, and save the resulting one calibration instance $(x_i, \hat{y}_{-i}, z_i)$, where $\hat{y}_{-i}$
  is the predicted value for $x_i$ using the model trained on the remaining
  data points. We repeat the process for all examples and we have the
  calibration dataset $\{(x_i, \hat{y}_{-i}, z_i) \}_{i=1}^N$
\end{itemize}
\vspace{-0.08in}
\section{Conclusion}
\vspace{-0.08in}
 
\label{section:conclusion}
In this paper, we described two measures for evaluating the calibration
capability of a binary classifier called \textit{maximum calibration error} (MCE) and
\textit{expected calibration error} (ECE). We also proved three theorems that
justify post processing as an approach for calibrating binary classifiers.
Specifically, we showed that by using a simple histogram-binning calibration
method we can improve the calibration of a binary classifier, in terms of $MCE$
and $ECE$, without sacrificing the discrimination performance of the classifier,
as measured in terms of $AUC$. The other contribution of this paper is to
introduce two extensions of the histogram-binning method that are based on
kernel density estimation and on the Dirichlet process mixture model. Our
experiments on simulated and real data sets showed that the proposed methods
performed well and are promising, when compared to two popular, existing
calibration methods. 

In future work, we plan to investigate the conjecture that histogram-binning
that uses equal frequency bins is a mini-max (or near mini-max) rate classifier,
as equal width binning is known to be. Our extensive experimental studies
comparing histogram binning with equal frequency and equal width bins provides support
that this conjecture is true. We also would like to show similar theoretical
proofs for kernel density estimation. Another direction for future research is
to extend the methods described in this paper to multi-class calibration
problems.

\section{Appendix A}
\label{section:Appendix}
In this appendix, we give the sketch of the proofs for the ECE and AUC bound
theorems mentioned in Section 3 (Calibration Theorems). It would be helpful to
review the Section 2 (Notations and Assumptions) of the paper before reading
the proofs.

\subsection{ECE Bound Proof} 
Here we show that using histogram binning calibration method, ECE converges to
zero with the rate of $O(\sqrt{\frac{B }{N}})$. Lest's define $E_i$ as the
expected calibration loss on bin $\tilde{B}_i$ for the histogram binning method.
Following the assumptions mentioned in Section 3 about MCE bound theorem, we
have $E_i = E(|e_i - o_i|)$. Also, using the definition of ECE and the notations
in Section 2, we can rewrite ECE as
the convex combination of $E_i$s. As a result, in order to bound ECE it suffices to show that all of its
$E_i$ components are bounded. Recall the concentration results proved in MCE
bound theorem in the paper we have:
\begin{equation}
\PP\{ |o_i-e_i| > \epsilon \} \leq 2 k_i e^{\frac{-N \epsilon^2}{2B}},
\label{eq:onebin}
\end{equation}

also let's recall the following two identities:

\begin{lemma}
\label{TH:Lemma1}
if X is a positive random variable then $E[X] = \int_0^\infty  \PP(X>t) dt$
\end{lemma}

\begin{lemma}
\label{TH:Lemma2}
 $\int_0^\infty  e^{-x^2} dx = \frac{\sqrt \pi}{2}$
\end{lemma}

Now, using the concentration result in Equation \ref{eq:onebin} and
applying the two above identities we can bound $E_i$ to write $E_i \leq C
\sqrt{\frac{B}{N}}$, where $C$ is a constant. Finally, since $ECE$ is the convex
combination of $E_i$'s we can conclude that
using histogram binning method, $ECE$ converges to zero with the rate of 
$O(\sqrt{\frac{B}{N}})$.

\subsection{AUC Bound Proof}
Here we show that the worst case AUC loss using histogram binning calibration
method would be at the rate of $O(\frac{1}{B})$. For
proving the theorem, let's first recall the concentration results for
$\hat{\eta}_i$ and $\hat{\theta}_i$.
Using Hoeffding's inequality we have the following:

\begin{align}
&\PP\{ |\hat{\theta}_i - \theta| \geq \epsilon \} \leq 2
e^{\frac{-2N\epsilon^2}{B}} \label{eq:theta-concenteration}\\
&\PP\{ |\hat{\eta}_i - \eta| \geq \epsilon \} \leq 2
e^{-2N\epsilon^2} 
\label{eq:eta-concenteration}
\end{align}

The above concentration inequalities show that with probability $1-\delta$
we have the following inequalities:

\begin{align}
&|\hat{\theta}_i-\theta_i| \leq \sqrt{\frac{B}{2N}\log (\frac{2}{\delta})}\\
&|\hat{\eta}_i-\eta_i| \leq \sqrt{\frac{1}{2N}\log (\frac{2}{\delta})}
\end{align}
The above results show that for the large amount of data with high probability,
$\hat{\eta_i}$ is concentrated around $\eta_i$ and $\hat{\theta}_i$ is
concentrated around around $\theta_i$.

Based on \cite{agarwal2006generalization} the empirical
AUC of a classifier $\phi(.)$ is defined as follow:

\begin{equation}
\hat{AUC}  = \frac{1}{m n} \sum_{i: z_i = 1} \sum_{j: z_j = 0} I(y_i > y_j) +
\frac{1}{2} I(y_i = y_j)
\label{eq:auc_emp} 
\end{equation}

Where $m$ and $n$ as mentioned in section [2] (assumptions and notations) in
main script are respectively the total number of positive and negative examples.
Computing the expectation of the equation \ref{eq:auc_emp} gives the actual AUC
as following:

\begin{equation}
AUC = Pr\{ y_i > y_j | z_i = 1, z_j = 0\} + \frac{1}{2} Pr\{y_i = y_j | z_i = 1, z_j = 0\}
\end{equation} 

It would be nice to mention that using the MacDiarmid concentration inequality
it is also possible to show that the empirical $\hat{AUC}$ is highly
concentrated around true $AUC$ \cite{agarwal2006generalization}.

 
Recall $p_{in}$ is the space of output of base classifier ($\phi$). 
Also, $p_{out}$ is the space of output of transformed probability estimate using
histogram binning.
Assume $B_1,\ldots, B_B$ are the non-overlapping bins defined on 
the $p_{in}$ in the histogram binning approach. Also, assume $y_i$ and $y_j$ are
the base classifier outputs for two different instance where $z_i = 1$ and $z_j
= 0$. In addition, assume $\hat{y}_i$ and $\hat{y}_j$ are respectively, the
transformed probability estimates for $y_i$ and $y_j$ using histogram binning
method.

Now using the above assumptions we can write the AUC loss of
using histogram binning method as following:

\begin{align}
AUC\_Loss &= AUC(y) - AUC(\hat{y})\\ 
&=\PP\{y_i>y_j | z_i = 1, z_j = 0 \} + \frac{1}{2} \PP\{y_i = y_j | z_i = 1, z_j
= 0 \} \notag\\
& - (\PP\{\hat{y}_i > \hat{y}_j | z_i = 1, z_j = 0 \} + \frac{1}{2}
\PP\{\hat{y}_i = \hat{y}_j | z_i = 1, z_j = 0 \})
\end{align}

By partitioning the space of uncalibrated estimates $p_{in}$ one can write the $AUC\_Loss$ as following:

\begin{align}
AUC\_Loss &= \sum_{K,L} (\PP\{y_i > y_j, y_i \in B_K, y_j \in B_L | z_i = 1,
z_j = 0 \} -\PP\{\hat{y}_i > \hat{y}_j, y_i \in B_K, y_j \in B_L | z_i = 1, z_j = 0 \}) \notag\\
		   &+ \sum_{K} (\PP\{y_i > y_j, y_i \in B_K, y_j \in B_K | z_i = 1, z_j = 0 \}
		   				+ \frac{1}{2} \PP\{y_i = y_j, y_i \in B_K, y_j \in B_K | z_i = 1, z_j=0 \}\notag \\ 
		   				&-\frac{1}{2} \PP\{\hat{y}_i = \hat{y}_j, y_i \in B_K,y_j \in B_K | z_i
		   				= 1, z_j = 0 \}) \notag\\
\label{eq:Loss_AUC1}					   
\end{align}

Where we make the following reasonable assumption that simplifies our
calculations: 
\begin{itemize}
  \item Assumption $1$ : $\hat{\theta}_i \neq \hat{\theta_j}$ if $i \neq j$ 
\end{itemize}


Now we will show that the first summation part in equation \ref{eq:Loss_AUC1} 
will be less than or equal to zero. Also, the second summation part 
will go to zero with the convergence rate of $O(\frac{1}{B})$.

\subsubsection*{First Summation Part}
Recall that in the histogram binning method the calibration estimate $\hat{y} = \hat{\theta}_K$ if $y \in B_K$.
Also, notice that if $y_i \in B_K$, $y_j \in B_L$ and $K > L$ then we have $y_i > y_j$ for sure.  
So, using the above facts we can rewrite the first summation part in equation \ref{eq:Loss_AUC1} as following:
\begin{align}
Loss_1 &= \sum_{K>L} \PP\{y_i \in B_K, y_j \in B_L | z_i = 1, z_j = 0 \}
		- \sum_{K,L} \PP\{\hat{\theta}_K > \hat{\theta}_L, y_i \in B_K, y_j \in B_L |
		z_i = 1, z_j = 0 \} \notag\\
\label{eq:Loss_AUC2}					   
\end{align} 

We can rewrite the above equation as following:

\begin{align}
Loss_1 = \sum_{K>L} &(\PP\{y_i \in B_K, y_j \in B_L | z_i = 1, z_j =0\}\notag\\
		&- \PP\{\hat{\theta}_K > \hat{\theta}_L, y_i \in B_K, y_j \in B_L |
		z_i = 1, z_j = 0 \} \notag\\
		&- \PP\{\hat{\theta}_L > \hat{\theta}_K, y_i \in B_L, y_j \in B_K |
		z_i = 1, z_j = 0 \}) \notag\\
\label{eq:Loss_AUC3}					   
\end{align}

Next by using the Bayes' rule and omitting the common denominators among the
terms we have the following:
\begin{align}
Loss_1 \propto \sum_{K>L} &\bigg( \PP\{z_i = 1, z_j=0|y_i \in B_K, y_j \in
							 		B_L\} \notag\\
						  &- \PP\{\hat{\theta}_K > \hat{\theta}_L, z_i = 1, z_j = 0| y_i  \in B_K,
								y_j \in B_L \} \notag\\
						  &- \PP\{\hat{\theta}_L > \hat{\theta}_K, z_i = 1, z_j = 0| y_i  \in B_L,
								y_j \in B_K \}\bigg) \times \PP\{y_i \in B_L, y_j \in B_K\}\notag\\
\label{eq:Loss_AUC1_1}					   
\end{align}

We next show that the term inside the parentheses in equation
\ref{eq:Loss_AUC1_1} is less or equal to zero by using the i.i.d. assumption and
the notations we mentioned in Section 2, as following:

\begin{align}
 Inside\_ Term(IT) &= (\theta_K (1-\theta_L) \notag\\
				&- \mathbb{I}\{\hat{\theta}_K > \hat{\theta}_L \}\theta_K(1-\theta_L)\notag\\
				&- \mathbb{I}\{\hat{\theta}_L >\hat{\theta}_K\}\theta_L(1-\theta_K))\notag\\
\label{eq:Loss_AUC1_2}					   
\end{align}

Now if we have the case $\hat{\theta}_K > \hat{\theta}_L$ then $IT$ term would
be exactly zero. if we have the case that $\hat{\theta}_L > \hat{\theta}_K$ then
the inside term would be equal to:
\begin{align}
 IT &= \theta_K (1-\theta_L) -\theta_L(1-\theta_K)\notag\\
                 &\simeq \hat{\theta}_K (1-\hat{\theta}_L)
                 -\hat{\theta}_L(1-\hat{\theta}_K)\notag\\
                 &\leq 0
\label{eq:Loss_AUC1_2}					   
\end{align}
where the last inequality is true with high probability which comes from the
concentration results for $\hat{\theta}_i$ and $\theta_i$ in equation
\ref{eq:theta-concenteration}.

\subsubsection*{Second Summation Part}
Using the fact that in the second summation part 
$\hat{y_i} = \hat{\theta}_K$ and $\hat{y_j} = \hat{\theta}_K$, we can rewrite
the second summation part as:

\begin{align}
Loss_2 &= \sum_{K} ( 
(\PP\{y_i > y_j, y_i \in B_K, y_j \in B_K | z_i = 1,z_j=0\} 
+ \frac{1}{2} \PP\{y_i = y_j, y_i \in B_K, y_j \in B_K | z_i = 1, z_j=0\})\notag \\
&- (\frac{1}{2} \PP\{y_i \in B_K, y_j \in B_K | z_i = 1, z_j = 0\}))\notag\\
	   &\leq \sum_{K} (\PP\{y_i \in B_K, y_j \in B_K | z_i = 1, z_j = 0 \}
					-\frac{1}{2} \PP\{y_i \in B_K, y_j \in B_K | z_i = 1, z_j = 0 \}) \notag\\
	   &= \frac{1}{2} \sum_{K} \PP\{y_i \in B_K, y_j \in B_K | z_i = 1, z_j = 0 \}) \notag\\
	   \label{eq:Loss2_1}			 			
\end{align}
 
Using the Bayes rule and iid assumption of data we can rewrite the equation
\ref{eq:Loss2_1} as following:

\begin{align}
Loss_2 &\leq \frac{1}{2}  \frac{\sum_{K} \PP\{ z_i = 1, z_j = 0 | y_i \in B_K, y_j \in B_K \}
						           \times \PP\{ y_i \in B_K, y_j \in B_K \}}
						   {\PP\{ z_i = 1, z_j = 0 \}}\notag \\
	   &= \frac{1}{2}  \frac{\sum_{K} \PP\{ z_i = 1, z_j = 0 | y_i \in B_K, y_j \in B_K \}
						           \times \PP\{ y_i \in B_K \} \PP\{ y_j \in B_K \}}
						   {\sum_{K,L} \PP\{ z_i = 1, z_j = 0 | y_i \in B_K, y_j \in B_L \}
						           \times \PP\{ y_i \in B_K \} \PP\{ y_j \in B_L \}} \notag\\
       &= \frac{1}{2}  \frac{\sum_{K} \PP\{ z_i = 1, z_j = 0 | y_i \in B_K, y_j \in B_K \}
						           \times \eta_K^2}
						   {\sum_{K,L} \PP\{ z_i = 1, z_j = 0 | y_i \in B_K, y_j \in B_L \}
						           \times \eta_K \eta_L}\notag \\
	   &= \frac{1}{2}  \frac{\sum_{K} \PP\{ z_i = 1, z_j = 0 | y_i \in B_K, y_j \in B_K \}}
						   {\sum_{K,L} \PP\{ z_i = 1, z_j = 0 | y_i \in B_K, y_j \in B_L \}}				           						            
	   \label{eq:Loss2_2}			 			
\end{align}
 
Where the last equality comes from the fact that $\eta_K$ and $\eta_L$ are
concentrated around their empirical estimates $\hat{\eta}_K$ and $\hat{\eta}_L$
which are equal to $\frac{1}{B}$ by construction (we build our histogram model
based on equal frequency bins).

Using the i.i.d. assumptions about the calibration samples, we can rewrite the
equation \ref{eq:Loss2_2} as following:
\begin{align}
Loss_2 &\leq \frac{\sum_{K} \PP\{ z_i = 1 | y_i \in B_K \}
							\PP\{ z_j = 0 | y_j \in B_K \}}
			      {2 \sum_{K} \PP\{ z_i = 1 | y_i \in B_K \} \times
				     \sum_{L} \PP\{ z_j = 0 | y_j \in B_L \}}\notag \\
	   &= \frac{\sum_{k=1}^B \theta_k (1-\theta_k)}
			      {2 \sum_{k=1}^B \theta_k \times
				     \sum_{l=1}^B (1-\theta_l)}\notag \\
	   &\leq \frac{1}{2B} 
	   \label{eq:Loss_AUC_Last}		     
\end{align}

Where the last inequality comes from the fact that the order of
$\{(1-\theta_1),\ldots, (1-\theta_B)\}$'s is completely reverse in comparison
to the order of $\{\theta_1,\ldots, \theta_B\}$ and applying Chebychev's Sum inequality.

\begin{theorem}
\emph{(Chebyshev's sum inequality)}
if $a_1 \leq a_2 \leq \ldots \leq a_n$ and $b_1 \geq b_2 \geq \ldots \geq b_n$
then   \\
$\frac{1}{n}\sum_{k=1}^{n}a_k b_k \leq (\frac{1}{n}\sum_{k=1}^n a_k
)(\frac{1}{n}\sum_{k=1}^n b_k)$
\end{theorem}

Now the facts we proved above about $Loss_1$ and $Loss_2$ in equations
\ref{eq:Loss_AUC_Last} and \ref{eq:Loss_AUC1_2} shows that the worst case
$AUC\_Loss$ is upper bounded by $O(\frac{1}{B})$ Using histogram binning
calibration method.

\begin{remark}

It should be noticed, the above proof shows that the worst case AUC loss at the
presence of large number of training data point is bounded by $O(\frac{1}{B})$.
However, it is possible that we even gain AUC power by using histogram binning
calibration method as we did in the case we applied calibration over the linear
SVM model in our simulated dataset.
\end{remark}

\bibliographystyle{amsplain}

\end{document}